\documentclass{article}
\usepackage{amssymb,amsmath}
\usepackage[margin=1.5in]{geometry}
\usepackage{enumerate}
\usepackage{natbib}
\usepackage{url} 
\usepackage{amsthm}
\usepackage{graphicx}

\usepackage{relsize}

\usepackage[ruled]{algorithm2e}

\usepackage{algorithmic}

\def\C{\mathcal{C}}
\def\X{\mathcal{X}}

\def\R{\mathbb{R}}

\def\x{\mathbf{x}}
\def\v{\mathbf{v}}
\def\u{\mathbf{u}}
\def\y{\mathbf{y}}
\def\z{\mathbf{z}}

\def\ncut{\mathrm{NCut}}

\def \erfc{\mathrm{erfc}}

\newtheorem{theorem}{Theorem}

\newtheorem{lemma}[theorem]{Lemma}

\newtheorem{remark}[theorem]{Remark}

\begin{document}



  \title{\bf Improving Spectral Clustering using the Asymptotic Value of the Normalised Cut}
  \author{David P. Hofmeyr\thanks{
    The author would like to thank Dr. Nicos Pavlidis and Ms. Katie Yates for their assistance with experiments, and Dr. Nicos Pavlidis for his valuable comments on this work.}\hspace{.2cm}\\
    Department of Statistics and Actuarial Science, Stellenbosch University}
  \maketitle

\begin{abstract}
Spectral clustering is a popular and versatile clustering method based on a relaxation of the normalised graph cut objective. Despite its popularity, however, there is no single agreed upon method for tuning the important scaling parameter, nor for determining automatically the number of clusters to extract. Popular heuristics exist, but corresponding theoretical results are scarce. In this paper we investigate the asymptotic value of the normalised cut for an increasing sample assumed to arise from an underlying probability distribution, and based on this result provide recommendations for improving spectral clustering methodology. A corresponding algorithm is proposed with strong empirical performance.
\end{abstract}

\noindent%

\section{Introduction}
\label{sec:intro}

Clustering is the problem of identifying relatively homogeneous groups (clusters) within a given dataset. Without a universal definition of what constitutes a cluster, multiple objectives have been proposed and a multitude of algorithms developed for each objective. A common and intuitive theme underlying most approaches is that data within the same cluster should be more similar than data in different clusters. A principled approach to addressing this relies on the construction of a similarity graph for the data, in which vertices correspond to data and edges are assigned a weight equal
the similarity between their vertex endpoints. A graph cut refers to the removal of a subset of edges from the graph which renders it disconnected. A cut of the similarity graph for which the edges removed have low total weight therefore induces a partition (a clustering) of the data for which the total similarity between data in different clusters is low. Such a cut, however, does not ensure that the similarity of data in the same cluster is comparatively high. Moreover, in practice the minimum weight cut tends to separate only small collections of points from the remainder, which may not constitute complete clusters~\citep{Luxburg2007}.
To account for these facts, normalisations of the graph cut objective have been proposed~\citep{Shi2000, Hagen1992}. Normalisation, however, renders the associated optimisation problem NP-hard~\citep{Wagner1993}. It has been shown that a continuous relaxation of the normalised graph cut problem can be solved using the eigenvectors of the so-called graph Laplacian matrix. It is this relaxation of the normalised graph cut problem which is given the name spectral clustering.

Spectral clustering has become extremely popular in recent years, for its versatility and strong performance in numerous application areas~\citep{ng2001spectral, Luxburg2007}. One of the appealing properties of spectral clustering is its ability to identify highly non-convex clusters.
%
%
Of crucial importance to the success of spectral clustering, however, is an appropriate measure of similarity to be used in the construction of the similarity graph. In addition, as with all clustering methods, determining automatically the number of clusters to extract from the data is a non-trivial task. In general the similarity between two points is determined by a decreasing function of the distance between them, so that pairs of data which are nearer in space tend to be assigned higher similarity than those which are more distant. For data belonging to $\R^d$, it is most common to define the similarity between $\x, \y \in \R^d$ as $k(\|\x - \y\|/\sigma)$, where $k:\R^+ \to \R^+$ is a kernel function, $\|\cdot \|$ is the Euclidean norm, and the parameter $\sigma > 0$ captures the scale of the data. The crucial factor then becomes the setting of the scaling parameter, $\sigma$. There is no single agreed upon method for setting this parameter. Some popular heuristics have been developed~\citep{ng2001spectral, Zelnik2004}, but few of these are supported by theory. If $\sigma$ is set too large, then the ability of spectral clustering to separate highly non-convex clusters is severely diminished. On the other hand, if $\sigma$ is set too small, then spectral clustering becomes highly susceptible to separating only outliers rather than complete clusters of data. Existing approaches for selecting the number of clusters rely either on the eigen-gaps of the graph Laplacian, or on the linear structure of the associated eigenvectors~\citep{Zelnik2004, Luxburg2007}.
%
%

In this paper we will use the asymptotic properties of the normalised cut to provide recommendations for improving spectral clustering methodology. Specifically, it will be shown that if the data are assumed to represent a sample of realisations of a continuous random variable with Lipschitz continuous density, then the (scaled) normalised cut measured across a smooth surface converges almost surely to the normalised integrated squared density over that surface. This is a specific case of the result of~\cite{trillos2015consistency}, for which we provide an alternative and what we believe to be more accessible proof. A similar investigation was conducted by~\cite{narayanan2006relation}, however our conclusions are substantially different. 
Relating normalised cuts to the probability distribution underlying the data helps to better understand the normalised cut objective, and therefore spectral clustering. This allows us to provide sound and justified recommendations for the practical implementation of spectral clustering methods. Importantly, it will be discussed how the minimum normalised cut solution tends to separate clusters by regions of low density, while at the same time separating high density regions in the underlying density function. These are key features in the non-parametric statistical approach to clustering, known as {\em density clustering}, in which clusters are associated with regions of high probability density surrounding the modes of the density function. This result is especially appealing as the density clustering objective provides a principled means for automatically determining the number of clusters; as the number of modes of the density function. Furthermore,
the conditions on the scaling parameter, $\sigma$,
required for convergence provide guidelines for establishing an order of magnitude for this parameter, and also dictate how it should decrease as sample sizes increase. Finally, this result provides some insight into the susceptibility of spectral clustering to outliers.


Based on these observations we propose a spectral clustering algorithm which we have found to exhibit strong performance in a range of applications. We are not aware of any existing algorithms which make use of the asymptotic properties of the normalised cut objective to improve clustering methodology. This is also the only spectral clustering algorithm of which we are aware which uses the structure of the clusters in the original data space, and not the eigenvector representation, to select the number of clusters.

The remainder of this paper is organised as follows. In Section~\ref{sec:spectral} a background on normalised cuts and spectral clustering will be provided. Following that in Section~\ref{sec:a.s.conv} we will investigate the asymptotic value of the normalised cut, and discuss in greater detail the practical implications of this result. We also propose an algorithm which makes use of these practical implications to automatically set the scaling parameter and select the number of clusters to extract.
%
%
%
The results of experiments on publicly available benchmark datasets are then presented in Section~\ref{sec:experiments}. A final discussion is provided in Section~\ref{sec:conclusions}.

\section{Normalised Cuts and Spectral Clustering} \label{sec:spectral}

In this section we provide a brief introduction to normalised cuts and the relaxed solution through spectral clustering. Note that the spectral clustering methodology is applicable to any data for which pair-wise similarities can be established, however we will focus only on Euclidean embedded data for this work. Let $\X = \{\x_1, ..., \x_n\}$ be a collection of data in $\R^d$. For each pair $\x_i, \x_j$ define their similarity as $k(\|\x_i - \x_j\|/\sigma)$, where $k:\R^+ \to \R^+$ is a kernel, and $\sigma > 0$ is a scaling parameter. A popular choice for $k$ is the Gaussian kernel, $k(x) = \exp(-x^2/2)$. Now, define the {\em affinity matrix} $A \in \R^{n \times n}$ s.t. 
\begin{equation}
A_{i,j} = \left\{\begin{array}{ll}
 k(\|\x_i - \x_j \|/\sigma), & i \not = j\\
0, & i=j,
\end{array}\right.
\end{equation}
and the corresponding {\em degree matrix} $D = \mbox{diag}\left(\sum_{j=1}^n A_{1, j}, ..., \sum_{j=1}^n A_{n, j}\right)$. The matrix $A$ uniquely defines a graph $\mathcal{G} = (\mathcal{V}, \mathcal{E})$, the similarity graph of the data, which has $n$ vertices (one per datum) and whose edges have weights given by the elements of $A$, i.e. $\mathcal{E}_{i,j} = A_{i,j}$.

\begin{remark}
It is important to note that some authors prefer to set the diagonal elements of $A$ to $k(0) = 1$. This is a matter of preference, and we have found that setting the diagonal of $A$ to zero improves performance. The asymptotic result we will show in the next section applies for both definitions of $A$.
\end{remark}

\noindent
Now, for a partition (clustering) of $\X$ into $\{\C_1, ..., \C_c\}$, the corresponding graph cut is given by,
\begin{equation}
\mbox{Cut}(\C_1, ..., \C_c) = \frac{1}{2}\sum_{i=1}^c \sum_{\substack{j,l: \x_j \in \C_i,\\ \x_l \not \in \C_i}} A_{j,l}.
\end{equation}

\noindent
The factor $\frac{1}{2}$ is in place since each edge removed to create the cut is counted twice in the above summation.
Notice that the total cut for such a partition can be re-expressed as $\sum_{i=1}^c \mbox{Cut}(\C_i, \X\setminus \C_i)$. A low value cut ensures that the total similarity between data in different clusters is low, however it does not necessarily result in clusters with high internal similarity. A preferable objective for clustering is the {\em normalised cut}~\citep{Shi2000}, defined as

\begin{equation}
\mbox{NCut}(\C_1, ..., \C_c\} = \frac{1}{2}\sum_{i=1}^c \frac{1}{\sum_{\x_j \in \C_i} D_{j,j}}\sum_{\substack{j,l: \x_j \in \C_i \\ \x_l \not \in \C_i}} A_{j,l} = \sum_{i=1}^c \frac{\mbox{Cut}(\C_i, \X \setminus \C_i)}{\sum_{\x_j \in \C_i} D_{j,j}}.
\end{equation}

\noindent
The quantity $\sum_{\x_j \in \C} D_{j,j}$, where $\C \subset \X$, is referred to as the {\em volume} of $\C$, denoted $\mbox{vol}(\C)$~\citep{Luxburg2007}.
Notice that a low value normalised cut will, in general, have a low value of $\mbox{Cut}(\C_i, \X\setminus \C_i)$ and a large value of $\mbox{vol}(\C_i)$, for each $i = 1, \dots, c$. Now,
recall that $D_{j,j} = \sum_{l=1}^n A_{j,l}$, and therefore,
\begin{align*}
\mbox{vol}(\C_i) &= \sum_{\substack{j,l: \x_j \in \C_i\\ \x_l \in \X}} A_{j,l}\\
&= \sum_{\substack{j,l: \x_j \in \C_i \\ \x_l \not \in \C_i}} A_{j,l} + \sum_{\substack{j,l: \x_j, \x_l \in \C_i}} A_{j,l}\\
&= \mbox{Cut}(\C_i, \X \setminus \C_i) + \sum_{\substack{j,l: \x_j, \x_l \in \C_i}} A_{j,l}.
\end{align*}
For $\mbox{Cut}(\C_i, \X \setminus \C_i)$ to be low and $\mbox{vol}(\C_i)$ to be high, we must therefore have $\sum_{\substack{j,l: \x_j, \x_l \in \C_i}} A_{j,l}$ being relatively high. This quantity represents the internal cluster similarity for cluster $\C_i$. The minimum normalised cut therefore tends to produce clustering solutions which have low between cluster similarity, and high within cluster similarity.

It has been shown that determining the optimal partition of $\X$ based on the normalised cut objective is NP-hard~\citep{Wagner1993}. However, a reformulation of the normalised cut objective, in terms of the so-called {\em graph Laplacian matrix}, leads to an intuitive relaxation of the problem~\citep{Shi2000, Luxburg2007}. The graph Laplacian may be defined as,
\begin{equation}
L = D^{-1/2}(D-A)D^{-1/2} = I_n - D^{-1/2}AD^{-1/2}.
\end{equation}
Now, for a partition of $\X$, as before, into $\{\C_1, \dots, \C_c\}$, define the matrix
\begin{equation}\label{eq:discrete}
V \in \R^{n \times c}: 
V_{i,j} = \left\{\begin{array}{ll}
\sqrt{1/\mbox{vol}(\C_j)}, & \x_i \in \C_j\\
0, & \x_i \not \in \C_j.
\end{array}\right.
\end{equation}
Then the minimum normalised cut problem for a fixed number of clusters, $c$, can be restated as~\citep{Luxburg2007},
\begin{align}\label{eq:prob1}
\min_{\C_1, \dots, \C_c} & \mbox{trace}(V^\top (D-A) V)\\
\nonumber
\mbox{s.t. } & V \mbox{ defined as in Eq.~(\ref{eq:discrete})}\\
\nonumber
& V^\top D V = I_c.
\end{align}
Relaxing the discreteness condition on the elements of $V$ in Eq.~(\ref{eq:discrete}) and setting $U = D^{1/2}V$ we arrive at
\begin{equation} \label{eq:prob2}
\min_{U \in \R^{n \times c}} \mbox{trace}(U^\top L U), \mbox{ s.t. } U^\top U = I_c,
\end{equation}
which is the canonical formulation of the problem of finding the eigenvectors of $L$ associated with its smallest $c$ eigenvalues. The approximate solution of the minimum normalised cut problem is therefore given by $D^{-1/2}U$, where $U \in \R^{n \times c}$ has as columns the $c$ eigenvectors of $L$ associated with its $c$ smallest eigenvalues. 

Notice that the optimal $V$ in problem~(\ref{eq:prob1}) has as columns scaled cluster indicator vectors for the clusters $\C_1, \dots, \C_c$, and hence the cluster assignment based on $V$ is immediate. The relaxed solution $D^{-1/2}U$ arising from problem~(\ref{eq:prob2}) does not necessarily take only discrete values. To obtain a final clustering solution, a simple clustering algorithm is applied to the rows of $D^{-1/2}U$. A popular choice for this final step is the $k$-means algorithm.

\section{Improving Spectral Clustering using the Asymptotic Value of the Normalised Cut} \label{sec:a.s.conv}

In Section~\ref{sec:spectral} the values of $\sigma$, the scale parameter, and $c$, the number of clusters, were assumed known. In practice determining these parameters is extremely challenging.
In this section we will investigate the asymptotic value of the normalised cut for an increasing sample assumed to arise from an underlying probability distribution. Specifically, we show that the asymptotic value of the normalised cut measured across a smooth surface converges almost surely to the normalised integrated squared density measured along that surface. This allows us to provide recommendations for automatically setting values of $\sigma$ and $c$. Based on these recommendations we propose an algorithm which we have found to exhibit strong empirical performance in a range of applications.

\subsection{The asymptotic value of the normalised cut}

The result presented here is a specific case of the result of~\cite{trillos2015consistency}; where we investigate the particular practical example of the normalised cut objective, using the Gaussian similarity kernel. By investigating this particular case we are able to provide a more accessible proof. A similar investigation into the asymptotic value of the normalised cut is given in~\cite{narayanan2006relation}, however the result we obtain herein is both more intuitive in relation to normalised cuts and is a preferable objective for clustering.
\begin{theorem}\label{thm:a.s.conv}
Let $\X = \{\x_1, \dots, \x_n\}$ be an i.i.d. sample of realisations of a continuous random variable, $\mathbf{X}$, over $\R^d$, with bounded and Lipschitz continuous probability density $p:\R^d \to \R^+$. Let $S$ be a smooth surface which partitions $\R^d$ into $S_1$ and $S_2$ such that $0 < \mathbb{P}(\mathbf{X} \in S_1) < 1$. Let $\{\sigma_n\}_{n=1}^\infty$ be a null sequence satisfying $n \sigma_n^{2d+2+\epsilon} \to \infty$ for some $\epsilon > 0$. Then,
$$
\frac{\sqrt{2\pi}}{\sigma_n}\ncut(\X \cap S_1, \X \cap S_2) \xrightarrow{a.s.} \oint_S p(\y)^2 d\y \left(\frac{1}{\int_{S_1} p(\y)^2 d\y} + \frac{1}{\int_{S_2} p(\y)^2 d\y}\right),
$$
as $n \to \infty$.
\end{theorem}
We use $\oint$ to distinguish integrals over $d-1$ dimensional surfaces lying in $\R^d$ from regular integrals over subsets of $\R^d$ or $\R$, for which we use $\int$.
\begin{remark}
The result in Theorem~\ref{thm:a.s.conv} can be shown to hold if the density, $p$, is only locally Lipschitz on a neigbourhood of $S$, and need not be globally Lipschitz. For ease of exposition, we prefer the above presentation.
\end{remark}
Notice that if we were able to maximise the minimum of the normalisation terms, $\int_{S_1}p(\y)^2d\y$ and $\int_{S_2} p(\y)^2d\y$, we would favour solutions in which there are regions within both $S_1$ and $S_2$ wherein the density takes values with especially large magnitude. This is because in maximising an integrated {\em squared} function, a solution in which the function takes large values, even over relatively small regions, would be preferred over a solution in which the function is of moderate magnitude over a relatively larger region. We expect, therefore, that the surface which minimises the normalised cut will (i) have low density along it, captured by the surface integral $\oint_S p(\y)^2 d\y$, and (ii) separate regions each containing points at which $p$ is large, as a result of normalising by the integrated squared density over $S_1$ and $S_2$. Specifically, if $p$ is multimodal then we would expect that the surface which minimises the normalised cut will tend to both pass through low density regions, and also separate the modes of $p$.  Both of these are key features in density clustering, in which clusters are associated with connected regions of high density surrounding the modes of the density, and are thus separated by regions of low density.

The proof of Theorem~\ref{thm:a.s.conv} is formed of the following two lemmas. Lemma~\ref{thm:support1} shows that the volume of the points arising within a measurable set, appropriately scaled, converges to the integrated squared density over that set. The intuition underlying this result is that for each $\x_i$, $\sum_{j =1}^n k(\|\x_i - \x_j\|/\sigma)$ is proportional to the kernel estimate of the density at $\x_i$, which we henceforth denote $\hat p(\x_i)$. The volume of the points arising in a measurable set, $M$, given by $\sum_{i: \x_i \in M} \sum_{j \not = i} k(\|\x_i - \x_j\|/\sigma)$, is therefore approximately proportional to a Monte Carlo estimate of $\int_M \hat p(\x) p(\x) d\x \approx \int_M p(\x)^2 d\x$. 
Lemma~\ref{thm:support2} then shows that the (scaled) cut across a smooth surface converges to the integrated squared density measured across it. The proof is more technical, and involves a series of approximations. Some useful ideas are borrowed from~\cite{narayanan2006relation}. The intuition in this case is that as the number of data increases, and so $\sigma_n$ approaches zero, the only kernel weights with a non-negligible contribution to the cut are those lying close to the surface, $S$. Each point on $S$ is therefore weighted proportionally to the number of pairs of data $\x \in S_1, \y \in S_2$ which have non-negligible similarity asymptotically. This weight is therefore proportional to squared density at the corresponding point on $S$. The complete proofs of Lemmas~\ref{thm:support1} and~\ref{thm:support2} are given in the appendix.

\begin{lemma}\label{thm:support1}
Let $\X = \{\x_1, \dots, \x_n\}$ be an i.i.d. sample of realisations of a continuous random variable, $\mathbf{X}$, over $\R^d$, with Lipschitz continuous probability density $p:\R^d \to \R^+$. Let $\{\sigma_i\}_{i=1}^\infty$ be a null sequence satisfying, $n\sigma_n^{2d+2+\epsilon} \to \infty$ for some $\epsilon > 0$. Let $M \subset \R^d$ be measurable. Then,
\begin{align*}
\frac{c_d}{n^2 \sigma_n^d}\sum_{i: \x_i \in M}\sum_{j \not = i}^n k\left(\frac{\|\x_i - \x_j\|}{\sigma_n}\right) \xrightarrow{a.s.} \int_{M} p(\x)^2 d\x, \mbox{ as } n \to \infty,
\end{align*}
where $c_d = (\int_{\R^d} k(\|\x\|)d\x)^{-1}$.
\end{lemma}
\noindent
Note that the rate of convergence of $\sigma_n$ to zero required for the above result is actually less strict than stated. In addition the choice of kernel is not limited to the Gaussian. Precise details of these requirements can be seen in~\cite{devroye1980strong}. However, the overall rate of convergence of $\sigma_n$ for Theorem~\ref{thm:a.s.conv} is dictated by the rate required for the following lemma. The proof of the following lemma also relies on the rotational invariance of the Gaussian kernel. In addition we will also make use of the distribution of the squared norm of the corresponding random variable, which is known in the case of the Gaussian distribution.

\begin{lemma}\label{thm:support2}
Let $\X = \{\x_1, \dots, \x_n\}$ be an i.i.d. sample of realisations of a continuous random variable, $\mathbf{X}$, over $\R^d$, with bounded Lipschitz continuous probability density $p:\R^d \to \R^+$. Let $S$ be a smooth surface which partitions $\R^d$ into $S_1$ and $S_2$ such that $0 < \mathbb{P}(\mathbf{X} \in S_1) < 1$. Let $\{\sigma_i\}_{i=1}^\infty$ be a null sequence satisfying $n\sigma_n^{2d+2+\epsilon} \to \infty$ for some $\epsilon > 0$. Then,
$$
\frac{c_d\sqrt{2\pi}}{n^2 \sigma_n^{d+1}}\sum_{i: \x_i \in S_1} \sum_{j: \x_j \in S_2} k\left(\frac{\|\x_i -  \x_j\|}{\sigma_n}\right) \xrightarrow{a.s.} \oint_S p(\y)^2 d \y \mbox{ as } n \to \infty,
$$
where $c_d = (\int_{\R^d} k(\|\x\|)d\x)^{-1}$.
\end{lemma}

\noindent
Lemmas~\ref{thm:support1} and~\ref{thm:support2} show immediately that Theorem~\ref{thm:a.s.conv} holds.
This is because a ratio of almost surely convergent sequences converges almost surely to the ratio of limits, provided the denominator limit is non-zero. This is ensured in this case by the assumption that $0 < \mathbb{P}(\mathbf{X} \in S_1) < 1$.

\begin{remark}
An alternative normalisation of the graph cut, which is referred to as Ratio Cut~\citep{Hagen1992}, is given as
\begin{equation*}
\mbox{Ratio Cut}(\C, \X\setminus \C) = \mbox{Cut}(\C, \X\setminus \C)\left(\frac{1}{\vert \C \vert} + \frac{1}{\vert \X \setminus \C \vert}\right).
\end{equation*}
Lemma~\ref{thm:support2} shows almost immediately that the Ratio Cut measured across $S$, appropriately scaled, converges almost surely to
\begin{equation*}
\oint_S p(\y)^2 d\y \left(\frac{1}{\mathbb{P}(\mathbf{X} \in S_1)} + \frac{1}{\mathbb{P}(\mathbf{X} \in S_2)}\right),
\end{equation*}
as $n \to \infty$.
\end{remark}

\subsection{Automating spectral clustering using the asymptotic properties of the normalised cut}

Recall that in density clustering clusters are associated with connected regions of high density surrounding the modes of a probability density function. In practice, however, the underlying distribution is unknown and must be estimated. This leads to an appealing interpretation of clusters as regions of high data density, separated by data sparse regions. The practical limitations of density clustering are that it becomes computationally prohibitive to compute these connected high density regions for large and high dimensional datasets, and that density estimation becomes less reliable as dimensionality increases~\citep{rinaldo2010generalized, menardi2014advancement}.

As discussed in the previous subsection, if the probability distribution underlying the data is multimodal, then the spectral clustering solution will tend to separate clusters by regions of low density, while also assigning high density regions, or more specifically modes, to different clusters. We expect therefore that a clustering solution arising from spectral clustering is likely to satisfy the objectives of density clustering. We can use this observation to determine an appropriate number of clusters to extract using spectral clustering, as follows. Suppose that $\{\C_1, \dots, \C_c\}$ represents a clustering solution arising from spectral clustering, for a number of clusters $c$. If $c$ is at most the number of modes of the probability density function then we expect that the cluster boundaries will pass through low density regions and that each cluster will contain at least one mode. On the other hand, if $c$ is greater than the number of modes we either have that a mode is split among multiple clusters or that at least one cluster contains only low density regions.
Motivated by this observation, we consider using spectral clustering to propose potential clustering solutions, and assessing their validity in the context of density clustering. Specifically, for a range of values $c$, we use spectral clustering to obtain a $c$-way clustering of $\X$. For each cluster we then determine if it is separated from the rest of the clusters only by low density regions. The largest value of $c$ for which all clusters are separated from the remainder then represents the correct number of clusters to extract. 

%
%
%
To determine if a cluster, $\C \subset \X$, is separated from the remainder we attempt to establish a path connecting $\C$ to $\X \setminus \C$ which does not pass through any relatively low density regions. If we fail to find such a path then we conclude that the cluster is separated. A point $\x \in \R^d$ has relatively low density if the estimated density at that point satisfies $\hat p(\x) < \lambda\min\{\max_{\y \in \C} \hat p(\y), \max_{\y \in \X \setminus \C} \hat p(\y)\}$, where $\lambda \in (0, 1]$ is a chosen setting. 
 We cannot, of course, consider all possible paths, and instead adopt the following heuristic.
Noticing that a path connecting a point in $\C$ to a point in $\X \setminus \C$ must pass through the boundary of $\C$, we need only to consider paths connecting the boundary of $\C$ to $\X\setminus \C$. To do this we define the boundary points of $\C \subset \X$ to be those elements which are nearest to some element of $\X \setminus \C$. That is,
\begin{equation}
\mbox{Boundary}(\C) = \bigcup_{\x \in \X \setminus \C} \mbox{arg}\min_{\y \in \C} d(\x, \y).
\end{equation}
Then for each $\x \in$ Boundary$(\C)$ we find the nearest point in $\X\setminus \C$, $\y = \mbox{arg}\min_{\z \in \X\setminus \C} d(\x, \z)$, and then search over the line segment $[\x, \y]$ for points of relatively low density. If there exists such a pair $\x, \y$ with no points of relatively low density between them then we conclude that $\C$ is not separate from the remaining clusters, as there is a path connecting them which avoids low density regions.
Pseudocode for the above procedure is given in Algorithm~\ref{alg:densep}. 

\begin{remark}
Recall that $\hat p(\x_i) \propto \sum_{j=1}^n k(\|\x_i-\x_j\|/\sigma) = D_{i,i}+1$. We therefore do not have to compute the density values from scratch, except when searching over line segments joing pairs of points in $\X$. All pairwise distances, which are used to find the boundary points of $\C$ and $\X\setminus \C$, are also already computed to obtain the spectral clustering solution.
\end{remark}

\begin{algorithm}
  \algsetup{linenosize=\tiny}
  \footnotesize
\begin{algorithmic}
\STATE $\hat p_{\mathrm{thresh}} \gets \min\{\max_{i: \x_i \in \C} D_{i,i} + 1, \max_{i: \x_i \in \X \setminus \C} D_{i,i} + 1\}$
\FOR{$\x \in$ Boundary$(\C)$}
\STATE $\y \gets \mbox{arg}\min_{\z \in \X \setminus \C} d(\x, \z)$
\IF{$\min_{\gamma \in [0, 1]} \sum_{i=1}^n k\left(\frac{\|\gamma \x + (1-\gamma) \y - \x_i\|}{\sigma}\right) \geq \lambda \hat p_{\mathrm{thresh}}$}
\STATE return FALSE
\ENDIF
\ENDFOR
\STATE return TRUE
\end{algorithmic}
\caption{is.density.separated$(\C, \X \setminus \C; \lambda)$ \label{alg:densep}}
\end{algorithm}


A more subtle point arising from Theorem~\ref{thm:a.s.conv}, and one which is a caveat concerning the above methodology, may be interpreted in the context of outliers. A dataset containing outliers may be thought of as having arisen from a mixture distribution comprising a main distribution, which contributes to the structure of interest and which represents the vast majority of probability mass; and a long-tailed distribution with low probability mass which produces a small number of outlying points. The length of the tail of this ``outlier" distribution dictates how susceptible, in general, an algorithm will be to the outliers. Theorem~\ref{thm:a.s.conv} helps to provide insight into the susceptibility of spectral clustering to outliers. For example, if the tail of the distribution is such that
$$
\lim_{b \to \infty} \frac{\oint_{\|\y \| = b} p(\y)^2 d\y}{\int_{\| \y \| > b} p(\y)^2 d\y} = 0,
$$
then the spectral clustering solution will tend to focus on separating only few points arising in the tail. We would not expect in this case that the density in regions separating these ``outlier clusters" from the remainder will be lower than the density {\em within} these outlier clusters. We identify outlier clusters solely based on their size, and therefore do not check if clusters containing fewer than a prespecified number of points are density separated from the remainder, and only focus on the more substantial clusters.
How these outliers are interpreted after the fact is open to a user. If the presence or absence of outliers is in itself a point of interest then these may be inspected separately. We simply merge outliers with their nearest non-outlier cluster, as for the context of this article this makes comparison with other flat clustering algorithms more straightforward.\\
%
%
\\
Finally, recall the requirement on the sequence of scaling parameters given in Theorem~\ref{thm:a.s.conv}, that $\lim_{n \to \infty} n\sigma_n^{2d + 2 + \epsilon} = \infty$ for some $\epsilon > 0$ is sufficient to obtain almost sure convergence of the normalised cut. Although this does not dictate a specific value of the scaling parameter for spectral clustering, it does suggest how $\sigma$ should decrease as the sample size increases. In particular setting $\sigma = s n^{-1/(2d+3)}$ is appropriate, where $s$ is some stable measure of scale which can be computed from the dataset. The value of $s$ should be almost surely bounded both above and away from zero to ensure convergence of the normalised cut.\\
\\
Combining the above points we arrive at an algorithm which we call Spectral Partitioning Using Density Separation (SPUDS), pseudocode for which can be found in Algorithm~\ref{alg:spuds}. For an initial number, $c_0$, we compute the spectral clustering solution with $c_0$ clusters. If all clusters whose sizes are above a chosen size (the non-outlier clusters) are density separated from the remainder, then the best solution must have at least $c_0$ clusters. We therefore repeatedly increase the number of clusters by one until the spectral clustering solution yields a non-outlier cluster which is not density separated from the remainder. The clustering solution for the largest number of clusters in which all non-outlier clusters are density separated is stored.

On the other hand, if in the clustering solution containing $c_0$ clusters there is a non-outlier cluster which is not density separated from the remainder, then the correct number of clusters is less than $c_0$. We therefore repeatedly decrease the number of clusters by one and store the first clustering solution in which all non-outlier clusters are density separated.

Finally, each outlier cluster in the stored clustering solution is merged with its nearest non-outlier cluster, and the resulting clusters are returned.\\
\\
One practical note on this method is as follows: if either $c_0$ is much smaller than the final number of clusters, or if outlying points are sufficiently distant that spectral clustering favours splitting off more and more of the tail of the data over producing more substantial clusters, then this method may run slowly. It can be accelerated by, rather than increasing $c$ by one with each iteration when seeking additional clusters, increasing it by a larger number. Then, when a solution is reached in which a non-outlier clusters is not density separated from the remainder, the value of $c$ can be fine tuned, decreasing it repeatedly by one as in the original algorithm.

\begin{algorithm}
  \algsetup{linenosize=\tiny}
  \footnotesize
Input: Data $\X$, initial cluster number $c_0$, density path threshold $\lambda$, scale measure function $s$, outlier cluster size threshold $\gamma$.
 \begin{algorithmic}
\STATE $\sigma \gets s(\X)n^{-1/(3+2d)}$
\STATE $A \gets A_{i,j} = \exp\left(-\frac{\|\x_i - \x_j\|^2}{2\sigma^2}\right), A_{ii} = 0$
\STATE $D \gets \mbox{diag}\left(\sum_{i=1}^n A_{i,1}, \dots, \sum_{i=1}^n A_{i,n}\right)$
\STATE $L \gets I_n - D^{-1/2}AD^{-1/2}$
\STATE $c \gets c_0$
\FOR{$i \in \{1, \dots c\}$}
\STATE $\u_i \gets$ eigenvector associated with the $i^{th}$ smallest eigenvalue of $L$
\ENDFOR
\STATE $U\gets [\u_1, \dots, \u_c]$
\STATE $\{\C_1, \dots, \C_c\} \gets k\mbox{-means}(D^{-1/2}U, c)$
\IF{$\forall i \in \{1, \dots, c\} : \vert \C_i \vert < \gamma \mbox{ {\bf  or} is.density.separated}(\C_i, \X\setminus \C_i; \lambda)$}
\REPEAT
\STATE $\{\C_1^\prime, \dots, \C_c^\prime\} \gets \{\C_1, \dots, \C_c\}$
\STATE $c \gets c+1$
\STATE $\u_{c} \gets$ eigenvector associated with the $c^{th}$ smallest eigenvalue of $L$
\STATE $U\gets [\u_1, \dots, \u_c]$
\STATE $\{\C_1, \dots, \C_c\} \gets k\mbox{-means}(D^{-1/2}U, c)$
\UNTIL{$\exists i \in \{1, \dots, c\} : \vert \C_i \vert \geq \gamma, \neg \mbox{is.density.separated}(\C_i, \X\setminus \C_i; \lambda)$}
\STATE $c \gets c-1$
\STATE $\{\C_1, \dots, \C_c\} \gets \{\C_1^\prime, \dots, \C_c^\prime\}$
\ELSE
\REPEAT
\STATE $c \gets c-1$
\STATE $U \gets [\u_1, \dots, \u_c]$
\STATE $\{\C_1, \dots, \C_c\} \gets k\mbox{-means}(D^{-1/2}U, c)$
\UNTIL{$\forall i \in \{1, \dots, c\} : \vert \C_i \vert < \gamma \mbox{ {\bf  or} is.density.separated}(\C_i, \X\setminus \C_i; \lambda)$}
\ENDIF
\STATE $O \gets \bigcup_{i: \vert \C_i \vert < \gamma} \{i\}$
\FOR{$i \in O$}
\STATE $i^\prime \gets \mbox{arg}\min_{j \in \{1, \dots, c\}\setminus O}d(\C_i, \C_j)$
\STATE $\C_{i^\prime}\gets \C_{i^\prime} \cup \C_{i}$ 
\ENDFOR
\RETURN $\bigcup_{i \in \{1, \dots, c\}\setminus O} \C_i$
\end{algorithmic}
  \caption{Spectral Partitioning Using Density Separation (SPUDS) \label{alg:spuds}}
\end{algorithm}

\section{Experimental Results} \label{sec:experiments}
 In this section we will investigate empirically the performance of the algorithm presented in the Section~\ref{sec:a.s.conv}. For all experiments we use the following settings. To measure the scale of each dataset we rely on the eigenvalues of its covariance matrix. Although these eigenvalues measure the linear scale of the data, where spectral clustering may be better informed by a more flexible scale measure, they benefit from simplicity and interpretability. In particular we define $s(\X) = \sqrt{\bar \epsilon}$, where $\bar \epsilon$ is the average of the largest $d^\prime$ eigenvalues of the covariance of $\X$, and $d^\prime$ is an estimate of the intrinsic dimension of $\X$ for which we use Kaiser's criterion~\citep{kaiser1960application}. We found this setting to work fairly consistently for datasets of moderate dimensionality, but for some higher dimensional examples Kaiser's criterion resulted in too small a value of $\sigma$. We therefore set $d^\prime$ to 20 if Kaiser's criterion selects a greater value. We distinguish outlier clusters from others as those which contain fewer than $n/200$ points. We set $c_0$, the initial number of clusters, to 30. Finally we set $\lambda$, the density threshold parameter, to $1$. This value corresponds to the strictest interpretation that clusters should not contain more than one mode. Note that for lower values of the scaling parameter, $\sigma$, the variance of the pointwise density estimates $\hat p(\x_i)$ is greater. In such cases it may be beneficial to accommodate some of this variability by setting $\lambda < 1$, thereby potentially connecting pairs of modes in $\hat p$ when their heights are not that well distinguished from the density between them.

In addition, we also experimented with the following existing algorithms in order to draw comparisons with the proposed method.
\begin{enumerate}
\item Self-Tuning Spectral Clustering (STSC): A fully automatic spectral clustering algorithm which uses a local scaling method to compute similarities and which selects the number of clusters which, through rotation, provides the best alignment of the eigenvectors of the Laplacian matrix with the canonical basis vectors~\citep{Zelnik2004}. We used the implementation provided by the authors available from {\tt http://www.vision.caltech.edu/lihi/Demos/SelfTuningClustering.html}.
\item Spectral Clustering using Cluster Distortion (SCCD): Spectral clustering which selects the scaling parameter as that which provides the minimum cluster distortion within the eigenvectors of the Laplacian matrix~\citep{ng2001spectral}. This method cannot automatically determine the number of clusters.
\item Alternative Spectral Clustering (Alt.SC): An automatic spectral clustering method which uses a similar scaling approach as STSC, but which uses strictly smaller scale values, and which selects the number of clusters using an adaptation of Bartlett's test for equal variances. We used the implementation in the {\tt R} package {\tt speccalt}~\citep{bruneau2013speccalt}, available from the {\tt CRAN} repository.
\item $k$-means: The ubiquitous $k$-means clustering algorithm is based on minimising the sum of squared distances between data and their assigned cluster's mean. We use the default implementation in {\tt R}, which is based on the algorithm of~\cite{hartigan1979algorithm}. To select the number of clusters we used the Gap Statistic~\citep{tibshirani2001estimating}, and set the maximum number of clusters to be twice the actual number.
\item Optimal Extraction of Clusters from Hierarchies (OCE): The algorithm of~\cite{campello2013framework} applied to cluster hierarchies constructed based on connectivity of clusters using single and average linkage as well as Ward's method. We report only those results based on Ward's method, as this method performed best on the data considered.
\item Gaussian Mixture Model (GMM): We used the method of~\cite{fraley2002model}, which uses the Bayesian Information Criterion to select the number of clusters, and the flexibility of the component covariance matrices.
\item DBSCAN: The method of~\cite{ester1996density} is an approximation of the density clustering solution for a kernel density estimate using the uniform kernel. The heuristics for setting the scale and level set parameters described by the authors did not produce sensible results. We therefore generated solutions for a range of settings and report the highest performance for each dataset. This therefore likely overestimates the true performance expected from DBSCAN. In addition we considered more flexible density clustering algorithms, but these were either incapable of handling the datasets due to size, or they did not render sensible solutions for any parameter settings tried.
\end{enumerate}
The following collection of benchmark datasets were used for comparison. Optical recognition of handwritten digits\footnote{\tt https://archive.ics.uci.edu/ml/datasets.html} (Opt. Digits), Pen-based recognition of handwritten digits$^1$ (Pen Digits), Multi-feature digits$^1$ (M.F. Digits), Statlog landsat satellite$^1$ (Satellite), Yale faces database B 40$\times$30\footnote{\tt https://cervisia.org/machine\_learning\_data.php}, Phoneme\footnote{\tt https://statweb.stanford.edu/$\sim$tibs/ElemStatLearn/data.html}, Smartphone based activity recognition$^1$ (Smartphone), Isolet$^1$, 
and Statlog image segmentation$^1$ (Image Seg.). 
We evaluate clustering solutions using the Normalised Mutual Information~\citep{strehl2002cluster}, which computes the proportion of shared information in the clustering result and the true clusters.
Results from these experiments are summarised in Table~\ref{tb:performance}. We see that the proposed algorithm obtains the highest or close to highest performance in all but one case (Satellite), in which it substantially underestimates the number of clusters. It frequently obtained the highest performance of all algorithms considered, and importantly consistently outperforms other spectral clustering algorithms.

\begin{table}[h!]
\centering
\caption{Clustering performance based on Normalised Mutual Information (NMI) on benchmark datasets. The highest performance in each case is highlighted in bold. \label{tb:performance}}
\scalebox{.93}{
\begin{tabular}{lc|cccccccc}
&& SPUDS & STSC & SCCD & Alt.SC & $k$-means & OCE & GMM & DBSCAN\\
\hline\hline
Opt. Digits 
& NMI & {\bf 0.79} & 0.73 & 0.01 & 0.01 & 0.68 & 0.59 & 0.63 & 0.56\\
(c=10, n=5620, d=64) & c & 13 & 9 & - & 3 & 19 & 4 & 9 & 8\\
\hline
Pen Digits
& NMI & 0.70 & 0.38 & 0.19 & {\bf 0.80} & 0.73 & 0.64 & 0.73 & 0.69\\
(c=10, n=10992, d=16)& c & 9 & 2 & - & 19 & 20 & 5 & 9 & 27\\
\hline
M.F. Digits
& NMI & 0.79 & 0.71 & {\bf 0.80} & 0.73 & 0.72 & 0.57 & 0.00 & 0.65\\
(c=10, n=2000, d=216)& c & 18 & 8 & - & 8 & 20 & 3 & 1 & 7\\
\hline
Satellite
& NMI & 0.40 & 0.39 & 0.62 & {\bf 0.66} & 0.60 & 0.38 & 0.55 & 0.51\\
(c=6, n=6435, d=36)& c & 3 & 2 & - & 7 & 11 & 3 & 9 & 5\\
\hline
Yale Faces
& NMI & {\bf 0.84}  & 0.04 & 0.73 & 0.28 & 0.70 & 0.81 & 0.78 & 0.54\\
(c=10, n=5850, d=1200)& c & 14 & 2 & - & 7 & 20 & 10 & 9 & 47\\
\hline
Phoneme
& NMI & 0.69 & 0.66 & {\bf 0.87} & 0.63 & 0.68 & 0.68 & 0.61 & 0.37\\
(c=5, n=4509, d=256)& c & 7 & 3 & - & 3 & 10 & 3 & 4 & 3\\
\hline
Smartphone
& NMI & 0.55* & {\bf 0.57} & 0.51 & 0.49 & 0.56 & 0.57 & 0.00 & 0.41\\
(c=12, n=10929, d=561)& c & 7* & 2 & - & 2 & 24 & 2 & 1 & 3\\
\hline
Isolet
& NMI & {\bf 0.72}* & 0.64 & 0.70 & 0.26 & 0.70 & 0.42 & 0.40 & 0.25\\
(c=26, n=6238, d=617)& c & 25* & 15 & - & 2 & 52 & 2 & 2 & 5\\
\hline
Image Seg.
& NMI & {\bf 0.68} & 0.42 & 0.03 & 0.01 & 0.61 & 0.46 & 0.62 & 0.50\\
(c=26, n=2310, d=19)& c & 6 & 3 & - & 3 & 14 & 2 & 8 & 5\\
\hline
\hline
Average && {\bf 0.68} & 0.50 & 0.50 & 0.43 & 0.66 & 0.57 & 0.48 & 0.50\\
Average Regret && {\bf 0.06} & 0.24 & 0.25 & 0.32 & 0.08 & 0.18 & 0.27 & 0.25\\
Average Rank && {\bf 1.44} & 4.06 & 3.22 & 4.17 & 2.28 & 3.78 & 3.94 & 5.11\\
\hline
\hline
\end{tabular}
}
{\scriptsize * indicates that the accelerated version of the algorithm, described at the end of Section~\ref{sec:a.s.conv}, was used. The number of clusters was increased by 10 with each iteration.}
\end{table}

In addition to the performance on the individual datasets, we also report the average performance over all datasets, as well as the average regret and average rank. The regret of an algorithm on a dataset is the difference between the best performing method on that dataset and the performance of the algorithm. The rank of an algorithm on a dataset is its position in the ordered set of performance values on that dataset. It is clear that overall the proposed algorithm is competitive with existing algorithms. Again the comparisons with existing spectral clustering algorithms is compelling. The only algorithm with similar performance overall is $k$-means. Note however that because of the high computation time associated with estimating the number of clusters via the Gap Statistic, we enforced an upper bound on the number of clusters which could be extracted. In most cases this upper bound was selected by the Gap. It is likely therefore that without such a reasonable upper bound this approach would even more significantly overestimate the number of clusters, thereby causing performance to deteriotate.

\section{Discussion} \label{sec:conclusions}

In this paper the asymptotic value of the normalised cut was investigated, and a new spectral clustering algorithm proposed which makes use of this asymptotic result. Speficically, based on the asymptotic value of the normalised cut we expect that clustering solutions arising from spectral clustering will both separate clusters by regions of low empirical density, and also assign high density regions to each cluster. This observation is used to automatically determine the number of clusters, as the largest value for which all clusters remain separated from the remainder by low density regions. Furthermore requirements on the scaling parameter to ensure convergence of the normalised cut are used to provide a data driven method for setting this important parameter. Experiments on a collection of benchmark datasets indicate that this new algorithm is competitive with state of the art clustering methods.

\section*{Appendix: Proofs of Lemmas~\ref{thm:support1} and~\ref{thm:support2}}

Throughout the remainder we will assume that the kernel $k$ is normalised so that $\int_{\R^d}k(\|\x\|) = 1$. Notice that this does not affect the value of the normalised cut, as the normalisation is applied in both the numerator and denominator terms.

\subsection*{Lemma~\ref{thm:support1}}

{\em
Let $\X = \{\x_1, \dots, \x_n\}$ be an i.i.d. sample of realisations of a continuous random variable, $\mathbf{X}$, over $\R^d$, with Lipschitz continuous probability density $p:\R^d \to \R^+$. Let $\{\sigma_i\}_{i=1}^\infty$ be a null sequence satisfying, $n\sigma_n^{2d+2+\epsilon} \to \infty$ for some $\epsilon > 0$. Let $M \subset \R^d$ be measurable. Then,
\begin{align*}
\frac{c_d}{n^2 \sigma_n^d}\sum_{i: \x_i \in M}\sum_{j \not = i}^n k\left(\frac{\|\x_i - \x_j\|}{\sigma_n}\right) \xrightarrow{a.s.} \int_{M} p(\x)^2 d\x, \mbox{ as } n \to \infty,
\end{align*}
where $c_d = (\int_{\R^d} k(\|\x\|)d\x)^{-1}$.
}

\begin{proof}
First notice that
$$
\frac{1}{n^2\sigma_n^d}\sum_{i:\x_i \in M}\sum_{j \not = i}k\left(\frac{\|\x_i - \x_j\|}{\sigma_n}\right) = \frac{1}{n^2\sigma_n^d}\sum_{i:\x_i \in M}\sum_{j =1}^n k\left(\frac{\|\x_i - \x_j\|}{\sigma_n}\right) + \mathcal{O}(n^{-1}\sigma_n^{-d})
$$
Now,
\begin{align*}
\frac{1}{n^2 \sigma_n^d} \sum_{i: \x_i \in M} \sum_{j=1}^n k\left(\frac{\| \x_i - \x_j\|}{\sigma_n}\right) &= \frac{1}{n}\sum_{i: \x_i \in M}^n p(\x_i) + \frac{1}{n}\sum_{i:\x_i \in M}\left(\frac{1}{n \sigma_n^d}\sum_{j=1}^n k\left(\frac{\|\x_i - \x_j\|}{\sigma_n}\right) - p(\x_i)\right).
\end{align*}
The first term is a standard Monte Carlo estimate of the integral $\int_{M} p(\x)^2 d\x$ and satisfies,
\begin{align*}
\left \vert \frac{1}{n}\sum_{i: \x_i \in M}^n p(\x_i) - \int_{M} p(\x)^2 d\x\right \vert \xrightarrow{a.s.} 0 \mbox{ as } n \to \infty.
\end{align*}
The second term represents the average error of a kernel estimator of $p$, and clearly satisfies,
\begin{align*}
\left \vert \frac{1}{n}\sum_{i:\x_i \in M}\left(\frac{1}{n \sigma_n^d}\sum_{j=1}^n k\left(\frac{\|\x_i - \x_j\|}{\sigma_n}\right) - p(\x_i)\right)\right \vert \leq \sup_{\x \in \R^d} \left \vert \frac{1}{n \sigma_n^d} \sum_{j=1}^n k\left(\frac{\|\x - \x_j\|}{\sigma_n}\right) - p(\x)\right \vert,
\end{align*}
where $\frac{1}{n\sigma_n^d} \sum_{j=1}^n k(\|\x - \x_j\|/\sigma_n)$ is the kernel based estimate of $p(\x)$. Now, since $p$ is Lipschitz continuous, and since the Gaussian kernel clearly satisfies the conditions of~\cite[Theorem 1]{devroye1980strong} for $\sigma_n \to 0, n\sigma_n^{2d+2+\epsilon} \to \infty$ for any $\epsilon > 0$,
the kernel estimate converges uniformly almost surely. Specifically,
\begin{align*}
\sup_{\x \in \R^d} \left \vert \frac{1}{n \sigma_n^d} \sum_{j=1}^n k\left(\frac{\|\x - \x_j\|}{\sigma_n}\right) - p(\x)\right \vert \xrightarrow{a.s.} 0 \mbox{ as } n \to \infty.
\end{align*}
Putting these together gives the result.
\end{proof}

\subsection*{Lemma~\ref{thm:support2}}

{\em
Let $\X = \{\x_1, \dots, \x_n\}$ be an i.i.d. sample of realisations of a continuous random variable, $\mathbf{X}$, over $\R^d$, with bounded Lipschitz continuous probability density $p:\R^d \to \R^+$. Let $S$ be a smooth surface which partitions $\R^d$ into $S_1$ and $S_2$ such that $0 < \mathbb{P}(\mathbf{X} \in S_1) < 1$. Let $\{\sigma_i\}_{i=1}^\infty$ be a null sequence satisfying $n\sigma_n^{2d+2+\epsilon} \to \infty$ for some $\epsilon > 0$. Then,
$$
\frac{c_d\sqrt{2\pi}}{n^2 \sigma_n^{d+1}}\sum_{i: \x_i \in S_1} \sum_{j: \x_j \in S_2} k\left(\frac{\|\x_i -  \x_j\|}{\sigma_n}\right) \xrightarrow{a.s.} \oint_S p(\y)^2 d \y \mbox{ as } n \to \infty,
$$
where $c_d = (\int_{\R^d} k(\|\x\|)d\x)^{-1}$.
}

\begin{proof}
To begin with, take $i \in \{1, \dots, n\}$, and let $\X^\prime = (\X \setminus \{\x_i\}) \cup \{\x_i^\prime\}$ for any $\x_i^\prime \in \R^d$.
$$
\left\vert \sum_{\x \in \X\cap S_1} \sum_{\y \in \X \cap S_2} k\left(\frac{\|\x-\y\|}{\sigma_n}\right) - \sum_{\x \in \X^\prime \cap S_1} \sum_{\y \in \X^\prime \cap S_2} k\left(\frac{\|\x-\y\|}{\sigma_n}\right)\right\vert \leq n,
$$
and hence the random variable $\nu = \frac{\sqrt{2\pi}}{n^2 \sigma_n^{d+1}}\sum_{i: \x_i \in S_1} \sum_{j: \x_j \in S_2} k\left(\frac{\|\x_i -  \x_j\|}{\sigma_n}\right)$ is difference bounded by $\frac{\sqrt{2\pi}}{n \sigma_n^{d+1}}$.
McDiarmid's inequality therefore ensures,
$$
\mathbb{P}\left(\left\vert \nu - \mathbb{E}[\nu] \right \vert > \epsilon \right) < \exp\left(-\frac{2\epsilon^2n \sigma_n^{2d+2}}{2\pi} \right).
$$
By assumption on the rate of convergence of $\sigma_n$ to zero, 
it is therefore sufficient to show that $\mathbb{E}[\nu] \to \oint_S p(\y)^2 d\y$ as $\sigma_n \to 0$. We henceforth drop the notational dependence on $n$ and write only $\sigma$. 

Establishing the convergence of $\mathbb{E}[\nu]$ will be achieved in a number of steps. To begin with consider that,
\begin{align*}
\mathbb{E}[\nu] = \frac{\sqrt{2\pi}}{\sigma}\mathbb{P}(\mathbf{X} \in S_1) \mathbb{P}(\mathbf{X} \in S_2) \mathbb{E}\left[\frac{1}{\sigma^d}k\left(\frac{\| \mathbf{Y} - \mathbf{Z}\|}{\sigma}\right)\right],
\end{align*}
where $\mathbf{Y}$ and $\mathbf{Z}$ are random variables representing the truncations of $\mathbf{X}$ in $S_1$ and $S_2$ respectively. Therefore,
\begin{align*}
\mathbb{E}\left[\frac{1}{\sigma^d} k\left(\frac{\| \mathbf{Y} - \mathbf{Z}\|}{\sigma}\right)\right] &= \int_{S_1} \int_{S_2} \frac{1}{\sigma^d} k\left(\frac{\| \y - \z\|}{\sigma}\right) \frac{p(\y)}{\mathbb{P}(\mathbf{X} \in S_1)} \frac{p(\z)}{\mathbb{P}(\mathbf{X} \in S_2)} d\z d\y \\
\Rightarrow  \mathbb{E}[\nu] &= \frac{\sqrt{2\pi}}{\sigma}\int_{S_1} \int_{S_2} \frac{1}{\sigma^d}k\left(\frac{\| \y - \z\|}{\sigma}\right) p(\y) p(\z) d\z d\y.
\end{align*}
Next we use the Lipschitz continuity of $p$ to show that the product $p(\y)p(\z)$ can be replaced by $p(\y)^2$ while inducing very small error. This is because the kernel weight $k(\|\y-\z\|/\sigma)$ is only large when $\z$ is close to $\y$, and hence $p(\z) \approx p(\y)$. Suppose $p$ has Lipschitz constant $L$. Then we have
\begin{align*}
p(\y) \int_{S_2} \frac{1}{\sigma^d}k\left(\frac{\|\y - \z\|}{\sigma}\right) (p(\y) - L \|\z - \y\|) d\z & \leq \int_{S_2} \frac{1}{\sigma^d}k\left(\frac{\| \y - \z\|}{\sigma}\right) p(\y)p(\z) d\z\\
& \leq p(\y) \int_{S_2} \frac{1}{\sigma^d}k\left(\frac{\|\y - \z\|}{\sigma}\right) (p(\y) + L \|\z - \y\|) d\z.
\end{align*}
Now, if $\y$ is s.t. $d(\y, S)\geq \sqrt{\sigma}$, then for small $\sigma$ we have
\begin{align*}
\int_{S_2}\frac{1}{\sigma^d}k\left(\frac{\|\y-\z\|}{\sigma}\right) \|\y-\z\|d\z & \leq \int_{\z : \| \z - \y \| \geq \sqrt{\sigma}} \frac{1}{\sigma^d}k\left(\frac{\|\y-\z\|}{\sigma}\right) \|\y-\z\|d\z\\
& = \mathcal{O}\left(\sqrt{\sigma}\left(1-\Phi(\sigma^{-0.5})\right)\right) = \mathcal{O}(\sigma^2),
\end{align*}
where $\Phi$ is the distribution function of the standard Gaussian random variable.
It is clear that in fact this quantity is of a far smaller order of magnitude than $\sigma^2$ as $\sigma$ approaches zero, however this is sufficient for our purposes. On the other hand, if $d(\y, S) < \sqrt{\sigma}$, then we instead have
\begin{align*}
\int_{S_2}\frac{1}{\sigma^d}k\left(\frac{\|\y-\z\|}{\sigma}\right) \|\y-\z\|d\z &\leq \int_{\R^d}\frac{1}{\sigma^d}k\left(\frac{\|\y-\z\|}{\sigma}\right) \|\y-\z\|d\z\\
& = \mathcal{O}(\sigma),
\end{align*}
as the latter integral is the expected distance of a standard $d$ dimensional Gaussian random variable from its mean, and is $\mathcal{O}(\sqrt{d}\sigma) = \mathcal{O}(\sigma)$.\\
\\
Next, using a similar geometric construction to that in~\citep{narayanan2006relation}, let $r$ be the radius of the largest ball which can be placed tangent to $S$ at any point and which intersects $S$ at a single point. Then, for $\y \in S_1$ s.t. $d(\y, S) < \min\{r, \sqrt{\sigma}\}$ let $\y^\prime$ be the nearest point to $\y$ lying on $S$ ($\y^\prime$ is unique since $d(\y, S) < r$). Define $H$ to be the hyperplane tangent to $S$ at $\y^\prime$ and let $H_1$ be the corresponding half space containing $\y$ and $H_2$ the half space not containing $\y$. We will show that as $\sigma$ approaches zero,
$$
\int_{S_2} \frac{1}{\sigma^d} k\left(\frac{\| \y - \z \|}{\sigma}\right) d\z = \int_{H_2} \frac{1}{\sigma^d} k\left(\frac{\| \y - \z \|}{\sigma}\right) d\z + \mathcal{O}(\sigma).
$$
To that end, let $B_1$ be the ball of radius $r$ tangent to $S$ at $\y^\prime$ and lying in $S_1$. Define $B_2$ similarly except $B_2 \subset S_2$. 
We clearly have
\begin{align*}
\left \vert \int_{S_2} \frac{1}{\sigma^d} k\left(\frac{\|\z - \y \|}{\sigma}\right) d\z\right. - & \left. \int_{H_2} \frac{1}{\sigma^d} k\left(\frac{\|\z - \y \|}{\sigma}\right) d\z \right \vert \leq \\
&  \int_{\R^d \setminus B_1} \frac{1}{\sigma^d} k\left(\frac{\|\z - \y \|}{\sigma}\right) d\z - \int_{B_2} \frac{1}{\sigma^d} k\left(\frac{\|\z - \y \|}{\sigma}\right) d\z.
\end{align*}
By translation so that the centre of $B_1$ lies at the origin, we have
$$
\int_{\R^d \setminus B_1} \frac{1}{\sigma^d} k\left(\frac{\|\z - \y \|}{\sigma}\right) d\z = 1 - \mathbb{P}\left( \|\mathbf{W}\|^2 \leq r^2\right),
$$
where $\mathbf{W} \sim \mathcal{N}(\y, \sigma^2 I)$. Therefore, the random variable $\|\mathbf{W}\|^2/\sigma^2$ is distributed $\chi^2$ with $d$ degrees of freedom and non-centrality parameter $\frac{1}{\sigma^2}(r - d(\y, S))^2$.
For brevity let $\eta = d(\y, S)$. We therefore have,
$$
\int_{\R^d \setminus B_1} \frac{1}{\sigma^d} k\left(\frac{\|\z - \y \|}{\sigma}\right) d\z = Q_{d/2}\left(\frac{r - \eta}{\sigma}, \frac{r}{\sigma}\right),
$$
where $Q_\nu(\alpha, \beta)$ is the Marcum $Q$-function. Similarly, we find that
$$
\int_{B_2} \frac{1}{\sigma^d} k\left(\frac{\|\z - \y \|}{\sigma}\right) d\z = 1 - Q_{d/2}\left(\frac{r + \eta}{\sigma}, \frac{r}{\sigma}\right).
$$
Now, for $z \in \mathbb{Z}$, it has been shown that~\citep{sun2008tight}
\begin{align*}
Q_{z + 0.5}(\alpha, \beta) &= \frac{1}{2} \mathrm{erfc}\left(\frac{\alpha + \beta}{\sqrt{2}}\right) + \frac{1}{2} \mathrm{erfc}\left(\frac{\beta - \alpha}{\sqrt{2}}\right) + \left(\frac{1}{2\pi \alpha \beta}\right)^{0.5} \exp\left(-\frac{\alpha^2+\beta^2}{2}\right) \sum_{k=0}^{z-1}\left(\frac{\beta}{\alpha}\right)^{k+0.5}\\
& \hspace{10pt} \times \sum_{r=0}^k\frac{(k+r)!}{r!(k-r)!(2\alpha \beta)^r}\left((-1)^r \exp\left(\alpha\beta\right) + (-1)^{k+1}\exp\left(-\alpha\beta\right)\right),
\end{align*}
where erfc is the complementary error function erfc$(x) = \frac{2}{\sqrt{\pi}}\int_{x}^\infty \exp(-z^2) dz$.
With some straightforward algebra, and for $\alpha, \beta \gg 0$ we can arrive at,
\begin{align*}
Q_{z + 0.5}(\alpha, \beta) &= \frac{1}{2} \mathrm{erfc}\left(\frac{\alpha + \beta}{\sqrt{2}}\right) + \frac{1}{2} \mathrm{erfc}\left(\frac{\beta-\alpha}{\sqrt{2}}\right) + \mathcal{O}(\beta^{z-1}\alpha^{-z}), & \beta \geq \alpha\\
Q_{z + 0.5}(\alpha, \beta) &= \frac{1}{2} \mathrm{erfc}\left(\frac{\alpha + \beta}{\sqrt{2}}\right) + \frac{1}{2} \mathrm{erfc}\left(\frac{\beta - \alpha}{\sqrt{2}}\right) + \mathcal{O}(\alpha\beta)^{-0.5}, & \mbox{otherwise}.
\end{align*}
For $d$ odd we therefore have,
\begin{align*}
 \int_{\R^d \setminus B_1} \frac{1}{\sigma^d} k\left(\frac{\|\z - \y \|}{\sigma}\right) d\z & - \int_{B_2} \frac{1}{\sigma^d} k\left(\frac{\|\z - \y \|}{\sigma}\right) d\z \leq \frac{1}{2}\left(\erfc\left(\frac{2r - \eta}{\sqrt{2}\sigma}\right) + \erfc\left(\frac{\eta}{\sqrt{2}\sigma}\right)\right)\\
& + \frac{1}{2}\left(\erfc\left(\frac{2r + \eta}{\sqrt{2}\sigma}\right) + \erfc\left(\frac{-\eta}{\sqrt{2}\sigma}\right)\right) - 1 + \mathcal{O}\left(\sigma\right).
\end{align*}
Now, erfc$(-x) + $erfc$(x) = 2$. Therefore,
\begin{align*}
 \int_{\R^d \setminus B_1} \frac{1}{\sigma^d} k\left(\frac{\|\z - \y \|}{\sigma}\right) d\z  - \int_{B_2} \frac{1}{\sigma^d} k\left(\frac{\|\z - \y \|}{\sigma}\right) d\z & \leq \frac{1}{2}\left(\erfc\left(\frac{2r - \eta}{\sqrt{2}\sigma}\right) + \erfc\left(\frac{2r + \eta}{\sqrt{2}\sigma}\right)  \right) + \mathcal{O}\left(\sigma\right)\\
& =  \mathcal{O}\left(\sigma\right)
\end{align*}
as $\sigma \to 0$, since $r > \eta$ and using the fact that erfc$(x) \leq \exp(-x)$ for $x>0$. For $d$ even we can simply use the above and the fact that $Q_{\nu}(\alpha, \beta)$ is increasing in $\nu$ for all fixed $\alpha, \beta$.
Specifically, $Q_{z-0.5}(\alpha, \beta) < Q_{z}(\alpha, \beta) < Q_{z+0.5}(\alpha, \beta)$ for $z \in \mathbb{Z}$. 

We therefore have shown that, for $\y$ s.t. $d(\y, S) < \min\{r, \sqrt{\sigma}\}$,
$$
\left \vert\int_{S_2} \frac{1}{\sigma^d} k\left(\frac{\| \y - \z \|}{\sigma}\right) d\z - \int_{H_2} \frac{1}{\sigma^d} k\left(\frac{\| \y - \z \|}{\sigma}\right) d\z \right \vert = \mathcal{O}(\sigma).
$$
On the other hand, if $d(\y, S) \geq \min\{r, \sqrt{\sigma}\}$, then we instead have
$$
\left \vert\int_{S_2} \frac{1}{\sigma^d} k\left(\frac{\| \y - \z \|}{\sigma}\right) d\z - \int_{H_2} \frac{1}{\sigma^d} k\left(\frac{\| \y - \z \|}{\sigma}\right) d\z \right \vert = \mathcal{O}(\sigma^2),
$$
as $\sigma \to 0$ since both terms in the left hand side are $\mathcal{O}(\sigma^m)$ for all $m \in \R$. The exponent of $\sigma$ is not crucial for the final result except that it must be strictly greater than $1$ in this instance.\\
\\
The next step simply relies on the recognition that, since $k$ is the Gaussian kernel and the Gaussian distribution is closed under marginalisation, we have
$$
\int_{H_2} \frac{1}{\sigma^d} k\left(\frac{\| \y - \z \|}{\sigma}\right) d\z = 1 - \Phi \left(\frac{\eta}{\sigma}\right).
$$
This is achieved by a transformation such that $\y$ lies at the origin and $H$ is defined by the equation $\z_1 = \eta$, and through marginalising out all other dimensions.

Bringing these approximations together, we therefore have for small $\sigma$,
\begin{align*}
\mathbb{E}[\nu] & = \frac{\sqrt{2\pi}}{\sigma} \int_{S_1} \int_{S_2} \frac{1}{\sigma^d}k\left(\frac{\|\y - \z\|}{\sigma}\right) p(\y) p(\z) d\z d\y\\
& = \frac{\sqrt{2\pi}}{\sigma}\int_{\substack{\y: \y \in S_1 \\ d(\y, S) < \sqrt{\sigma}}} p(\y)\left(p(\y)\int_{S_2} \frac{1}{\sigma^d}k\left(\frac{\|\y - \z\|}{\sigma}\right) d\z + \mathcal{O}(\sigma)\right)d\y\\
& \hspace{10pt} + \frac{\sqrt{2\pi}}{\sigma}\int_{\substack{\y: \y \in S_1 \\ d(\y, S) \geq \sqrt{\sigma}}} p(\y)\left(p(\y)\int_{S_2} \frac{1}{\sigma^d}k\left(\frac{\|\y - \z\|}{\sigma}\right) d\z + \mathcal{O}(\sigma^2)\right)d\y\\
& = \frac{\sqrt{2\pi}}{\sigma}\int_{\substack{\y: \y \in S_1 \\ d(\y, S) < \sqrt{\sigma}}} p(\y)\left(p(\y)\left(1-\Phi\left(\frac{d(\y, S)}{\sigma}\right) + \mathcal{O}(\sigma)\right) + \mathcal{O}(\sigma)\right)d\y\\
& \hspace{10pt} + \frac{\sqrt{2\pi}}{\sigma}\int_{\substack{\y: \y \in S_1 \\ d(\y, S) \geq \sqrt{\sigma}}} p(\y)\left(p(\y)\left(1-\Phi\left(\frac{d(\y, S)}{\sigma}\right)+ \mathcal{O}(\sigma^2)\right) + \mathcal{O}(\sigma^2)\right)d\y\\
& = \frac{\sqrt{2\pi}}{\sigma}\int_{\substack{\y: \y \in S_1 \\ d(\y, S) < \sqrt{\sigma}}} \left(p(\y)^2\left(1-\Phi\left(\frac{d(\y, S)}{\sigma}\right)\right) + p(\y)(1+p(\y))\mathcal{O}(\sigma)\right)d\y\\
& \hspace{10pt} + \frac{\sqrt{2\pi}}{\sigma}\int_{\substack{\y: \y \in S_1 \\ d(\y, S) \geq \sqrt{\sigma}}} \left(p(\y)^2\left(1-\Phi\left(\frac{d(\y, S)}{\sigma}\right)\right) + p(\y)(1+p(\y))\mathcal{O}(\sigma^2)\right)d\y\\
& = \frac{\sqrt{2\pi}}{\sigma}\int_0^{\sqrt{\sigma}} \oint_{\substack{\y: \y \in S_1\\ d(\y, S) = \eta}} \left(p(\y)^2 \left(1-\Phi\left(\frac{\eta}{\sigma}\right)\right) +p(\y)(1+p(\y))\mathcal{O}(\sigma)\right)d\y d\eta\\
& \hspace{10pt} + \frac{\sqrt{2\pi}}{\sigma}\int_{\sqrt{\sigma}}^\infty \oint_{\substack{\y: \y \in S_1\\ d(\y, S) = \eta}} \left(p(\y)^2 \left(1 - \Phi\left(\frac{\eta}{\sigma}\right)\right)  + p(\y)(1+p(\y))\mathcal{O}(\sigma^2)\right) d\y d\eta\\
& = \frac{\sqrt{2\pi}}{\sigma} \int_{0}^\infty \oint_{\substack{\y: \y \in S_1\\ d(\y, S) = \eta}} p(\y)^2 \left(1 -  \Phi\left(\frac{\eta}{\sigma}\right)\right) d\y d\eta + \mathcal{O}(\sqrt{\sigma}).
\end{align*}
Now, since $p$ is bounded and Lipschitz, we know that 
$$
f:\R^+ \to \R^+: \eta \mapsto \oint_{\substack{\y: \y \in S_1\\ d(\y, S) = \eta}} p(\y)^2d\y
$$
is Lipschitz. Therefore,
\begin{align*}
\mathbb{E}[\nu] & = \frac{\sqrt{2\pi}}{\sigma} \int_0^\infty \left(1 - \Phi\left(\frac{\eta}{\sigma}\right)\right) \left(\oint_S p(\y)^2 d\y + \mathcal{O}(\eta)\right) d\eta + \mathcal{O}(\sqrt{\sigma})\\
& = \oint_S p(\y)^2 d\y \left(\frac{\sqrt{2\pi}}{\sigma} \int_0^\infty \left(1 - \Phi\left(\frac{\eta}{\sigma}\right)\right)\right) + \frac{\sqrt{2\pi}}{\sigma}\int_0^\infty \left(1 - \Phi\left(\frac{\eta}{\sigma}\right)\right)\mathcal{O}(\eta) d\eta + \mathcal{O}(\sqrt{\sigma})\\
& = \oint_S p(\y)^2d\y + \mathcal{O}(\sqrt{\sigma}),
\end{align*}
since $\frac{\sqrt{2\pi}}{\sigma}\int_0^\infty (1-\Phi(\eta/\sigma)) d\eta = 1$ and $\frac{1}{\sigma}\int_0^\infty (1-\Phi(\eta/\sigma))\eta d\eta = \mathcal{O}(\sigma)$. This proves the result.
\end{proof}

\bibliographystyle{Chicago}


\end{document}